\def\year{2021}\relax

\documentclass{article}
\usepackage{kr}

\usepackage{times}
\usepackage{soul}
\usepackage{url}
\usepackage[hidelinks]{hyperref}
\usepackage[utf8]{inputenc}
\usepackage[small]{caption}
\usepackage{graphicx}
\usepackage{amsmath}
\usepackage{amsthm}
\usepackage{booktabs}
\usepackage{algorithmic}
\usepackage{xspace}
\usepackage[linesnumbered,ruled,vlined]{algorithm2e}
\usepackage{paralist}
\usepackage{multirow}

\newtheorem{theorem}{Theorem}
\newtheorem{observation}{Observation}

\newtheorem{lemma}{Lemma}
\newtheorem{definition}{Definition}

\newcommand{\tuple}[1]{\ensuremath{\left \langle #1 \right \rangle }}
\newcommand{\pre}{\textit{pre}}
\newcommand{\params}{\textit{params}}
\newcommand{\eff}{\textit{eff}}
\newcommand{\name}{\textit{name}}
\newcommand{\type}{\textit{type}}
\newcommand{\cnf}{\textit{CNF}}
\newcommand{\conj}{\textit{Conj}}
\newcommand{\realm}{\ensuremath{M^*}\xspace}
\newcommand{\liftf}{F}
\newcommand{\liftl}{L}
\newcommand{\lifta}{A}
\newcommand{\sam}{\textit{SAM}\xspace}
\newcommand{\sgam}{\textit{SGAM}\xspace}
\newcommand{\bindings}{\textit{bindings}}
\newcommand{\iseff}{\text{IsEff}}
\newcommand{\ispre}{\text{IsPre}}

\usepackage{color}

\title{Safe Learning of Lifted Action Models\footnote{This Arxiv paper is an extended version of a paper with the same title that have been accepted to the International Conference on Principles of Knowledge Representation and Reasoning (KR), 2021.}}
\author{%
Brendan Juba $^1$\and
Hai S. Le$^1$\and
Roni Stern$^{2,3}$\\
\affiliations
$^1$Washington University in St. Louis, USA\\
$^2$Palo Alto Research Center, USA\\
$^3$Ben Gurion University of the Negev, Israel\\

\emails
\{bjuba, hsle\}@wustl.edu,
rstern@parc.com,
sternron@post.bgu.ac.il
}

\begin{document}

\maketitle

\begin{abstract}
Creating a domain model, even for classical, domain-independent planning, is a notoriously hard knowledge-engineering task. 
A natural approach to solve this problem is to learn a domain model from observations. 
However, model learning approaches frequently do not provide safety guarantees: the learned model may assume actions are applicable when they are not, and may incorrectly capture actions' effects. 
This may result in generating plans that will fail when executed. 
In some domains such failures are not acceptable, due to the cost of failure or inability to replan online after failure. 
In such settings, all learning must be done offline, based on some observations collected, e.g., by some other agents or a human. Through this learning, the task is to generate a plan that is guaranteed to be successful. 
This is called the model-free planning problem. 
Prior work proposed an algorithm for solving the model-free planning problem in classical planning. 
However, they were limited to learning grounded domains, and thus they could not scale. 
We generalize this prior work and propose the first safe model-free planning algorithm for lifted domains. 
We prove the correctness of our approach, and provide a statistical analysis showing that the number of trajectories needed to solve future problems with high probability is linear in the potential size of the domain model.
We also present experiments on twelve IPC domains showing that our approach is able to learn the real action model in all cases with at most two trajectories. 

\end{abstract}

\section{Introduction}



In classical domain-independent planning, a \emph{domain model} is a model of the environment and how the acting agent can interact with it. The domain model is given in a formal planning description language such as STRIPS~\cite{fikes1971strips} or the Planning Domain Definition Language (PDDL)~\cite{mcdermott1998pddl}. Domain-independent planning algorithms (planners) use the domain model to  generate a plan for achieving a given goal condition from a given initial state. 
Creating a domain model, however, is a notoriously hard knowledge-engineering task. 


To overcome this modeling problem, a variety of learning methods have been proposed. 
Model-free Reinforcement Learning (RL)
avoids the need for a domain model by learning directly how to act by performing actions and observing their outcomes. 
Other learning approaches aim to learn a world model from past observations, and use that model to solve future planning problems~\cite{amir2008}. Notably, Asai and Muise~\shortcite{asai2020learning} recently demonstrated this approach can even learn a PDDL model directly from (non-symbolic) images. 
However, all these approaches \textbf{permit the generation of failing actions}, i.e., actions that are either not applicable in the current state or do not achieve the intended effects.  
In some domains, this is acceptable and the agent simply incorporates such experiences and updates its internal model to improve future executions. 
In other domains, however, failing action must be avoided and only safe actions are allowed.  
This occurs when execution failure is too costly, or the agent cannot replan due to limited computational capabilities.
The problem of finding a safe plan, i.e., a plan that will not fail, without possessing a domain model, is called \textbf{safe model-free planning}~\cite{stern2017efficientAndSafe}. 
In safe model-free planning, instead of a domain model the planning agent is given a set of trajectories from plans that were executed in the past in the same domain (e.g., by a different agent or a human).

Stern and Juba~\shortcite{stern2017efficientAndSafe} proposed a sound algorithm for safe model-free planning, i.e., an algorithm that generates plans that do not fail, provided that the environment is actually captured by a (grounded) STRIPS model. 
However, their algorithm is not complete, i.e., it may not return a plan for a solvable planning problem. 
Nevertheless, they proposed a PAC-style model of learning to plan, in which completeness may be relaxed to ``approximate completeness'' with respect to the distribution of problems observed during training. They thus bounded the probability of encountering problems their model cannot solve, given a number of trajectories quasi-linear in the number of actions.
However, their positive result is limited to \emph{grounded} domain models, that is, domains that are not defined by \emph{lifted}, i.e., parameterized, actions and fluents. 
The size of a grounded domain model can be arbitrarily larger than its corresponding lifted domain model. In particular, a single lifted action can yield a number of grounded actions that grow polynomially with the number of objects in the domain, with the number of parameters of the lifted action as its exponent. 
In addition, learning a grounded domain model limits the  generalization possible between different groundings of the same lifted domain. For example, a grounded action model for a blocksworld domain with 8 blocks cannot be used to solve problems for a blocksworld domain with 9 blocks.  
This significantly limits the applicability of Stern and Juba's algorithm.


In this work, we overcome these limitations by presenting an algorithm that efficiently solves safe model-free planning problems for lifted domains. 
The key component of this approach is an algorithm that learns a \emph{safe action model}, which is a model of the agent's possible actions that is consistent with the underlying, unknown, domain  model. We call this algorithm \emph{Safe Action Model (SAM)} Learning.

Two versions of SAM learning are presented. The first 
may be used when each object is only ever bound to one action parameter at a time in the example trajectories.
We prove that this version is sound, and when the actions and fluents have bounded arity, we can guarantee that the action model is sufficient with high probability after observing a number of trajectories that is linear in the possible size of the lifted model.
Importantly, the number of trajectories needed depends only on the size of this lifted model, and is independent of the number of objects in the domain, in contrast to Stern and Juba's algorithm. 
We also observed efficient learning experimentally on twelve domains from the International Planning Competition (IPC)~\cite{ipc}: SAM learning is able to learn the real action model for all cases with at most two trajectories.
Finally, we discuss a more general version of SAM learning, for the case where multiple arguments are bound to the same object in some trajectories.










Our work also revisits the algorithm of Stern and Juba, and shows that it can be interpreted as solving a kind of knowledge-based learning task, similar to \emph{inductive logic programming}~\cite{muggleton1994inductive}, using the STRIPS axioms as background knowledge. We show in particular that the obtained model is the action model with the largest possible set of feasible plans (i.e., least constrained) that can be proven safe with the given trajectories, and in this sense is the \emph{strongest} safe action model. We show that our algorithms for lifted domains also enjoy this property.



\section{Background and Problem Definition}

Let $O$ be a set of objects and let $T$ be a set of types. 
Every object $o\in O$ is associated with a type $t\in T$ denoted $\type(o)$. 
For example, in the logistics domain from the International Planning Competition (IPC)~\cite{ipc} there are types \emph{truck} and \emph{location} and there may be objects $t_1$ and $t_2$ that represent two different trucks and two objects $l_1$ and $l_2$ that represent two different locations. 

\subsection{Lifted and Grounded Literals}
A \emph{lifted fluent} $\liftf$ is a pair $\tuple{\name, \params}$ representing a possible relation over typed objects, where  $\name$ is a symbol and $\params$ is a list of types. 
We denote the name of $\liftf$ and its parameters by $\name(\liftf)$ and $\params(\liftf)$ respectively, and $arity(\liftf,t)$ denotes the number of type-$t$ parameters. 
For example, in the logistics domain $at(?truck, ?location)$ is a lifted fluent that represents some truck ($?truck$) is at some location ($?location$). 
A \emph{binding} of a lifted fluent $\liftf$ is a function $b: \params(\liftf)\rightarrow O$ 
mapping every parameter of $\liftf$ to an object in $O$ of the indicated type. 
A \emph{grounded fluent} $f$ is a pair $\tuple{\liftf, b}$ where $\liftf$ is a lifted fluent 
and $b$ is a binding for \liftf. 
To \emph{ground} a lifted fluent $\liftf$ with a binding $b$ means to 
apply the
create a relation 
over the objects in the image of $b$ that match the relation over the corresponding parameters. 
We call this relation a \emph{grounded fluent} or simply a fluent, and denote it by $f$. 
In our logistics example, for $\liftf=at(?truck, ?location)$ and $b=\{?truck: truck1, ?location: loc1\}$ 
the corresponding grounded fluent $f$ is $at(truck1, loc1)$.
The term \emph{literal} refers to either a fluent or its negation. 
The definitions of binding, lifted, and grounded fluent transfer naturally to literals. 
A \emph{state} of the world is a set of grounded literals that, for every grounded fluent, either includes that fluent or its negation.

\subsection{Lifted and Grounded Actions}
A lifted action $\lifta\in \mathcal{A}$ is a pair $\tuple{\name, \params}$ 
where $\name$ is a symbol and $\params$ is a list of types, 
denoted $\name(\lifta)$ and $\params(\lifta)$, respectively, and $arity(\lifta,t)$ denotes the number of type-$t$ parameters. 
The action model $M$ for a set of actions $\mathcal{A}$ 
is a pair of functions $\pre_M$ and $\eff_M$ that map every action in $\mathcal{A}$ to its preconditions and effects. 
To define the preconditions and effects of a lifted action, 
we first define the notion of a \emph{parameter-bound literal}. 
A \emph{parameter binding} of a lifted literal $\liftl$ and an action $\lifta$ is a function $b_{\liftl,\lifta}: \params(\liftl)\rightarrow \params(\lifta)$ that maps every parameter of $\liftl$ to a parameter in $\lifta$. 
A \emph{parameter-bound literal} $l$ for the lifted action $\lifta$ is a 
pair of the form $\tuple{\liftl,b_{\liftl,\lifta}}$ where $b$ is a parameter binding of $\liftl$ and $\lifta$. 
$\pre_M(\lifta)$ and $\eff_M(\lifta)$ are sets of parameter-bound literals for $\lifta$. 

A \emph{binding} of a lifted action $\lifta$ is defined like a binding of a lifted fluent, i.e., a function $b:\params(\lifta)\rightarrow O$. 
A \emph{grounded action} $a$ is a tuple $\tuple{\lifta, b_\lifta}$ where $\lifta$ is a lifted action and $b_\lifta$ is a binding of $\lifta$. 
The preconditions of a grounded action $a$ according to the action model $M$, denoted $\pre_M(a)$, is the set of grounded literals created by taking every parameter-bound literal $\tuple{\liftl, b_{\liftl,\lifta}}\in \pre_M(\lifta)$ and grounding $\liftl$ with the binding $b_\lifta\circ b_{\liftl,\lifta}$. 
The effects of a grounded action $a$, denoted $\eff_M(a)$, are defined in a similar manner. 
The grounded action $a$ can be applied in a state $s$ iff $\pre_M(a)\subseteq s$. 
The outcome of applying $a$ to a state $s$ according to action model $M$, denoted $a_M(s)$, is a new state that contains all literals in $\eff_M(a)$ and all the literals in $s$ such that their negation is not in $\eff_M(a)$. 
Formally:
\begin{equation}\small
    a_M(s)=\{ l | l\in s \wedge \neg l\notin \eff_M(a) \vee l\in \eff_M(a) \} 
\end{equation}
We omit $M$ from $a_M(s)$ when it is clear from the context.
The outcome of applying a sequence of grounded actions $\pi=(a_1,\ldots a_n)$ to a state $s$ is the state $s'=a_n(\cdots a_1(s)\cdots)$. 
A sequence of actions $a_1,\ldots, a_n$ can be applied to a state $s$ 
if for every $i\in 1,\ldots,n$ the action $a_i$ is applicable in the state 
$a_{i-1}(\cdots a_1(s)\cdots)$. 

\begin{definition}[Trajectory]
A trajectory $T=\tuple{s_0, a_1, s_1, \ldots a_n, s_n}$ is an alternating sequence of states $(s_0,\ldots,s_n)$ and actions $(a_1,\ldots,a_n)$ that starts and ends with a state.
\end{definition}
The trajectory created by applying $\pi$ to a state $s$ is 
the sequence $\tuple{s_0, a_1, \ldots, a_{|\pi|}, s_{|\pi|}}$ such that 
$s_0=s$ and for all $0<i\leq |\pi|$, $s_i=a_i(s_{i-1})$. 
In the literature on learning action models~\cite{wang1994learning,wang1995learning,walsh2008efficient,stern2017efficientAndSafe,arora2018review}, 
it is common to represent a trajectory 
$\tuple{s_0, a_1, \ldots, a_{|\pi|}, s_{|\pi|}}$
as a set of triples 
$\big\{\tuple{s_{i-1},a_i,s_i}\big\}_{i=1}^{|\pi|}$.
Each triple $\tuple{s_{i-1},a_i,s_i}$ is called an \emph{action triplet},  and the states $s_{i-1}$ and $s_i$ are referred to as the pre- and post- state of action $a_i$. 
We denote by $\mathcal{T}(a)$ the set of all action triplets in the trajectories in $\mathcal{T}$ that include the grounded action $a$. $\mathcal{T}(\lifta)$ is defined for all action triplets that contain actions that are groundings of the lifted action $\lifta$.  

\subsection{Domains and Problems}

A classical planning \textbf{domain} is defined by a tuple 
$\tuple{T, \mathcal{F}, \mathcal{A}, M}$
where $T$ is a set of types, 
$\mathcal{F}$ is a set of lifted fluents, 
$\mathcal{A}$ is a set of lifted actions, 
and $M$ is an action model for $\mathcal{A}$.
A classical planning \textbf{problem} is defined by a tuple $\tuple{D, O,  s_I, G}$ where $D$ is a classical planning domain;  
$O$ is a set of objects; 
$s_I$ is the start state, i.e., the state of the world before planning;  
and $G$ is a set of grounded literals that define when the goal has been found. 
A \textbf{solution} to a planning problem is a sequence of grounded actions that can be applied to $s_I$ and if applied to $s_I$ results in a state $s_G$ that contains all the grounded literals in $G$. 
Such a sequence of grounded actions is called a \emph{plan}. 
The trajectory of a plan starts with $s_I$ and ends with a goal state $s_G$ (where $G\subseteq s_G$). 
The \emph{safe model-free planning} problem~\cite{stern2017efficientAndSafe} is defined as follows. 
\begin{definition}[Safe model-free planning]
Let $\Pi=\tuple{\tuple{T, \mathcal{F}, \mathcal{A}, \realm}, O, s_I, G}$ be a classical planning problem and let $\mathcal{T}=\{\mathcal{T}_1,\ldots, \mathcal{T}_m\}$ be a set of trajectories 
for other planning problems in the same domain. 
The input to a safe model-free planning algorithm is the tuple $\tuple{T,O, s_I, G, \mathcal{T}}$ and the desired output is a plan $\pi$ that is a solution to $\Pi$. We denote this safe model-free planning problem as $\Pi_{\mathcal{T}}$. 
\label{def:safe-model-free-planning}
\end{definition}
We refer to the action model $\realm$ as the real action model. 
The trajectories in $\mathcal{T}$ share the same domain as $\Pi$, 
and thus they have been generated by applying actions from $\mathcal{A}$ 
and following the action model specified in $\realm$. 
However, these trajectories may start in states that are not from $s_I$, 
may end in states that do not satisfy $G$, 
and may consider a set of objects that is different from $O$.  
Safety is captured in Definition~\ref{def:safe-model-free-planning} by requiring that the output plan $\pi$ is a \textbf{sound plan} for $\Pi$. That is, $\pi$ is applicable and ends up reaching a state that satisfies the goal.  
The main challenge is that the problem-solver -- the agent -- needs to find a sound plan to $\Pi$ but it is not given the set of fluents, actions, and action model of the domain ($\mathcal{F}$, $\mathcal{A}$, and \realm, respectively).

In this work, we make the following simplifying assumptions. 
Actions have deterministic effects, 
the agent has complete observability, 
and when the agent observes a grounded action $a=\tuple{\lifta, b_a}$, it is able to discern that $a$ is the result of grounding $\lifta$ with $b_a$. 
Similarly, if it observes a state with a grounded fluent $f=\tuple{\liftf, b_f}$, it is able to discern that $f$ is the result of grounding $\liftf$ with $b_f$. Also, we assume that actions' preconditions and effects are conjunctions of literals, as opposed to more complex logical statements, and we do not currently consider conditional effects of actions. 
These assumptions are reasonable when planning in digital/virtual environments, such as video games, or environments that have been instrumented with reliable sensors, such as warehouses designed to be navigated by robots~\cite{li2020lifelong}. 
Later in this paper, we discuss approaches to relax these assumptions and apply our work to a broader range of environments. 

\section{Conservative Planning in Grounded Domains}

Our approach for solving the model-free planning problem in lifted domains builds on the {conservative planning} approach proposed by Stern and Juba~\shortcite{stern2017efficientAndSafe} for grounded domains. 
Thus, we first describe their approach. 
This is done in a slightly different framing, which allows us to present a new theoretical property regarding the strength of the learned action model.  

\subsection{Inference Rules for Grounded Domains}
In a grounded domain, 
a state is a set of literals, 
and so are the preconditions and effects of all actions. 
That is, there is no notion of lifted literals of actions.

First, we define the notion of a consistent action model following the semantics of classical planning. 
\begin{definition}[Consistent Action Model]\label{def:consistent}
An action model $M$ is consistent with a set of trajectories $\mathcal{T}$ 
if for every action triplet $\tuple{s,a,s'}\in \mathcal{T}(a)$ 
it holds that:
\begin{compactenum}
    \item All preconditions are satisfied: $\forall l\in \pre(a) \forall s: l\in s$
    \item All effects are satisfied: $\forall l\in \eff(a) \forall s': l\in s'$
    \item Frame axioms\footnote{This means literals only change as a result of action effects.} hold: $\forall (l\notin \eff(a) \wedge l\notin s) \rightarrow l\notin s'$
\end{compactenum}
\end{definition}
\noindent
The contrapositives of the conditions in the above definition can be interpreted as inference rules as follows. 
\begin{observation}[Inference rules for grounded domains]\label{obs:sam-rules-grounded}
For any action triplet $\tuple{s,a,s'}$ it holds that:
\begin{compactitem}
    \item Rule 1 [not a precondition].  $\forall l \notin s: l \notin \pre(a)$
    \item Rule 2 [not an effect].  $\forall l \notin s': l \notin \eff(a)$
    \item Rule 3 [must be an effect].  $\forall l \in s'\setminus s: l \in \eff(a)$
\end{compactitem}
\end{observation}
\noindent
So, Rule 1 states that a literal that is not in a pre-state cannot be a precondition. 
Rule 2 states that a literal that is not in a post-state cannot be an effect. 
Rule 3 states that a literal that is in the post-state but not in the pre-state, must be an effect. 
Since this is just a restatement of the definition of a consistent action model, these rules precisely characterize the action models that are consistent with a given set of traces.

In the fully observable deterministic world of classical planning, every action model that is not consistent with the given set of trajectories is false, and the set of consistent action models must contain the real action model. However, some of the consistent action models are different from the real action model, and plans generated with them may yield a failure, e.g., trying to apply an action in a state in which not all preconditions hold. 

\begin{definition}[Safe Action Model]
\label{def:safe_action_model}
An action model $M'$ is safe with respect to an action model $M$ iff for every state $s$ 
and grounded action $a$ 
it holds that
\begin{equation}
\small
    \pre_{M'}(a)\subseteq s \rightarrow 
    \Big(\pre_M(a)\subseteq s \wedge 
    a_{M'}(s)=a_M(s)\Big)
    \label{eq:safe_action_model}
\end{equation}
\end{definition}
\noindent
In words, Definition~\ref{def:safe_action_model} says that if action model $M'$ is safe w.r.t. $M$ then 
for every state $s$ and action $a$, if $a$ is applicable in $s$ according to $M'$ then
(1) $a$ is also applicable in $s$ according to $M$, 
and (2) applying $a$ to $s$ results in the same state according to both action models. 
We say that an action model is safe if it is a safe action model w.r.t. the real action model \realm.

Observe that any plan generated by a planner given a safe action model
must also be a sound plan according to \realm. 
The \emph{conservative planning} approach~\cite{stern2017efficientAndSafe} for safe model-free planning is based on this observation. 
In conservative planning, we first learn from the given set of trajectories an action model $M$ that is safe w.r.t.\ \realm, and then apply an off-the-shelf planner to generate plans using $M$. 
To learn such a safe action model, 
Stern and Juba~\shortcite{stern2017efficientAndSafe} proposed the following algorithm. 
First, assume every action $a$ has all literals as its preconditions and no literals as its effects. 
Then, iterate over every action triplet in $\mathcal{T}(a)$ 
and apply the rules in Observation~\ref{obs:sam-rules-grounded} to remove incorrect preconditions and to add effects. 
We refer to this algorithm hereafter as the \emph{Safe Grounded Action-Model (SGAM) Learning} algorithm, and discuss its theoretical properties. 

\subsection{Theoretical Analysis}

\begin{theorem}[SGAM Learning is sound~\cite{stern2017efficientAndSafe}]
SGAM learning produces a safe action model. 
\label{thm:sam-safe-grounded}
\end{theorem}
The main limitation of using a safe action model $M_{\textit{safe}}$ is that it may be \emph{weaker} than the real action model (\realm), in the sense that there may be states in which an action $a$ is applicable according to \realm, but not applicable according to $M_{\textit{safe}}$. 
Consequently, there may be planning problems that are solvable with \realm but not with $M_{\textit{safe}}$. 
This is stated in a more formal and general below. 
\begin{definition}[Strength of Action Models]
If there exists a trajectory that is consistent with $M'$ but not with $M$, then we say that $M$ is weaker than $M'$.
If no such trajectory exists then we say that $M$ is at least as strong as $M'$. 
\label{def:weakness}
\end{definition}
If $M$ is at least as strong as $M'$ then given enough computation time, every planning problem that is solvable with $M'$ is also solvable with $M$. 
Alternatively, if $M'$ is weaker than $M$ then there may be planning problems that cannot be solved using $M'$ but can be solved using $M'$. 
Next, we complement Theorem~\ref{thm:sam-safe-grounded} by showing that the action model returned by SGAM learning is at least as strong as every safe action model that is consistent 
with the given trajectories. 




\begin{theorem}[The Strength of SGAM Learning]
Let $M_{SGAM}$ be the action model created by SGAM learning given the set of trajectories $\mathcal{T}$. 
$M_{SGAM}$ is at least as strong as any action model $M'$ that is safe and consistent with $\mathcal{T}$. 
\label{thm:sam-learning-complete-grounded}
\end{theorem}
\begin{proof}
Consider an action model $M'$, which is safe and consistent with $\mathcal{T}$. 
Let $a$ be an action and $s$ be a state such that $a$ is applicable in $s$ according to $M'$, i.e., $\pre_{M'}(a)\in s$. 
Since $M'$ is safe w.r.t.\ \realm, then 
$\pre_{\realm}(a)\subseteq s$
and $a_{M'}(s)=a_{\realm}(s)$. 
By construction of $M_\sgam$, if a literal $l$ is a precondition of $a$ according to $M_\sgam$, 
then it has appeared in the pre-state of all action triplets in $\mathcal{T}(a)$. 
Thus, there exists a consistent action model in which $l$ is a precondition of $a$ 
and this action model may be the real model. 
Therefore, since $M'$ is safe it follows that $\pre_{M'}(a)\subseteq \pre_{M_\sgam}(a)$, 
and thus $a$ is applicable in $s$ according to $M_\sgam$, 
i.e., $\pre_{M_\sgam}(a)\in s$. 
Since $M_\sgam$ is safe, 
$a_{M_\sgam}(s)=a_{\realm}(s)=a_M'(s)$.
Thus, every trajectory consistent with  $M'$ will also be consistent with $M_\sgam$.
\end{proof}

While the action model returned by SGAM is at least as strong as any other safe action model, it may still be weaker than the real action model. Consequently, conservative planning for model-free planning is bound to be sound but incomplete---it generates plans that are sound but it may fail to generate plans for some solvable planning problems.

A statistical analysis showed that under some assumptions, the number of trajectories SGAM learning needs 
to learn a safe action model that can solve most problems is quasilinear in the number of actions in the domain~\cite{stern2017efficientAndSafe}. 
However, the number of grounded actions in a \emph{lifted domain} can be quite large: the number of grounded actions that are groundings of a single lifted action grows polynomially with the number of objects in the domain (exponentially in the number of parameters). 
On the other hand, in a lifted domain, the real action model is assumed to be defined by lifted actions. 
This enables us to generalize SGAM learning across multiple groundings of the same lifted action, eliminating the dependence on the number of objects in the number of trajectories needed to learn a useful safe action model. We describe this in the next section.

\section{Conservative Planning for Lifted Domains}


In this section, we describe a conservative planning approach for safe model-free planning in lifted domains, 
which is based on a novel generalization of SGAM learning to lifted domains. We refer this algorithm as simply SAM learning. 
To describe SAM learning, 
we denote by $\bindings(b_\lifta, b_\liftl)$ the set of all 
parameter bindings $b_{\liftl, \lifta}$ that satisfy the following
\begin{equation}
\small
    b_\lifta\circ b_{\liftl,\lifta} = b_\liftl.
    \label{eq:inferbind}
\end{equation}





\subsection{Inference Rules for Lifted Domains}
The core of our algorithm is the following generalization of Observation~\ref{obs:sam-rules-grounded}, 
defining what observing an action triplet with a grounded action $\tuple{\lifta, b_\lifta}$ entails for the lifted action $\lifta$.
\begin{observation}\label{obs:sam-rules-lifted-general}
For any action triplet $\tuple{s, \tuple{\lifta, b_\lifta}, s'}$
\begin{compactitem}
    \item Rule 1 [not a precondition].  
    $\forall \tuple{\liftl, b_\liftl} \notin s:$
    \begin{equation}\small
     \forall b\in \bindings(b_\lifta, b_\liftl): 
     \tuple{\liftl,b} \notin \pre(\lifta)
    \end{equation}
    \item Rule 2 [not an effect].  
    $\forall \tuple{\liftl, b_\liftl} \notin s'$:
    \begin{equation}\small
     \forall b\in \bindings(b_\lifta, b_\liftl): 
     \tuple{\liftl,b} \notin \eff(\lifta)
    \end{equation}
    \item Rule 3 [an effect]. 
    $\forall \tuple{\liftl,b_\liftl} \in s'\setminus s:$
    \begin{equation}\small
    \exists b\in\bindings(b_\lifta, b_\liftl): 
    \tuple{\liftl,b} \in \eff(\lifta)
    \end{equation}
    I.e., in ILP terminology, the grounded literal $\tuple{\liftl,b_\liftl}$ is \emph{subsumed} by some $\tuple{\liftl,b} \in \eff(\lifta)$.
\end{compactitem}
\end{observation}
\noindent
For much of this paper, we make the following assumption:
\begin{definition}[Injective Action Binding]
In every grounded action $\tuple{\lifta, b_\lifta}$, the binding $b_\lifta$ is an injective function, 
i.e., every parameter of $\lifta$ is mapped to a different object. 
\label{def:injective}
\end{definition}
Under this assumption, for every pair of bindings $b_\liftl$ and $b_\lifta$ 
there exists a unique $b_{\liftl,\lifta}$ that satisfies Eq.~\ref{eq:inferbind}. 
This binding is obtained by inverting $b_\lifta$, i.e., 
\begin{equation}\small
    \bindings(b_\lifta, b_\liftl) = \{ (b_\lifta)^{-1}\circ b_\liftl \}.
    \label{eq:inferbind-inverse}
\end{equation}
where $(b_\lifta)^{-1}$ maps an object $o$ to the parameter of $\lifta$ that $b_\lifta$ maps to $o$. 
That is, each grounded literal and action that appears in a trajectory is essentially a renaming of the corresponding parameters in the lifted literals and actions by objects; the grounded literals and actions are \emph{OI-subsumed} by the lifted literals and actions \cite[Section 5.5.1]{deraedt08}.
Equation~\ref{eq:inferbind-inverse} simplifies the inference rules given in Observation~\ref{obs:sam-rules-lifted-general}. 
In particular, the ``an effect'' rule (Rule 3) becomes
\begin{equation}\small
    \forall \tuple{\liftl,b_\liftl} \in s'\setminus s: 
     \tuple{\liftl,(b_\lifta)^{-1}\circ b_\liftl} \in \eff(\lifta).
\end{equation}


\subsection{SAM Learning for Lifted Domains}

\begin{algorithm}[t]
\footnotesize
\DontPrintSemicolon
\SetKwInOut{Input}{Input}\SetKwInOut{Output}{Output}
\Input{$\Pi_\mathcal{T} = \tuple{T,O, s_I, G, \mathcal{T}}$}
\Output{An action model that is safe w.r.t. the action model that generated $\mathcal{T}$}
\BlankLine
    $\mathcal{A'}\gets$ all lifted actions observed in $\mathcal{T}$\\
    \ForEach{lifted action $\lifta\in \mathcal{A'}$ \nllabel{sam:line:loop-start}}{
        $\eff(\lifta)\gets\emptyset$ \\ 
        $\pre(\lifta)\gets $ all parameter-bound literals \nllabel{line:init_pre} \\
        \ForEach{$(s, \tuple{\lifta,b_\lifta}, s')\in \mathcal{T}(\lifta)$}{
            \ForEach{$\tuple{L,b_{\liftl,\lifta}}\in\pre(\lifta)$}{
                \If{$\tuple{L, b_\lifta\circ b_{\liftl,\lifta}}\notin s$}{
                    Remove $\tuple{L,b_{\liftl,\lifta}}$ from $\pre(\lifta)$ \nllabel{line:remove-pre}
                }
            }
            \ForEach{$\tuple{L,b_\liftl}\in s'\setminus s$}{
                $b_{\liftl, \lifta}\gets \tuple{\liftl, (b_\lifta)^{-1}\circ b_\liftl}$\big) \nllabel{line:infer-effect}\\
                Add $\tuple{\liftl, b_{\liftl,\lifta}}$ to $\eff(\lifta)$ \nllabel{line:add-eff}\\
            }
            \nllabel{sam:line:loop-end}
        }
    } 
    \Return $(\pre, \eff)$ \nllabel{line:return}
\caption{Safe Action-Model (SAM) Learning}\label{safeplan}
\label{alg:sam}
\end{algorithm}

We now present our SAM Learning algorithm for lifted domains in Algorithm~\ref{safeplan}. 
For every lifted action $\lifta$ observed in some trajectory, 
we initially assume that $\lifta$ has no effects and all possible parameter-bound literals are its preconditions (line~\ref{line:init_pre} in Algorithm~\ref{alg:sam}).\footnote{It is possible 
to initialize the preconditions of every lifted action 
to the pre-state of one of the action triplets in which it is used.}  
Then, for every action triplet $(s, \tuple{\lifta, b_\lifta}, s')$ with this lifted action, we remove from the preconditions of $\lifta$ every parameter-bound literal $\tuple{\liftl, b_{\liftl,\lifta}}$ that is not satisfied in the current pre-state (Rule 1 in Observation~\ref{obs:sam-rules-lifted-general}). 
Then, for every grounded literal $\tuple{\liftl, b_\liftl}$ that holds in the post-state $s'$ and not in $s$, we add a corresponding effect to $\lifta$ (Rule 3 in Observation~\ref{obs:sam-rules-lifted-general}). 
Note that Rule 2 in Observation~\ref{obs:sam-rules-lifted-general} is not needed since we initialize the set of effects of every action to be an empty set. 

\begin{theorem}
Given a set of trajectories $\mathcal{T}$, SAM learning (Algorithm~\ref{alg:sam}) runs in time 
\begin{small}
\[\mathcal{O}\Big(\sum_{\lifta\in \mathcal{A}}|\mathcal{T}(\lifta)|\sum_{\liftf\in\mathcal{F}}\prod_{t\in T}arity(\lifta,t)^{arity(\liftf,t)}\Big)\]
\end{small}
\end{theorem}
\begin{proof}
For every action $\lifta\in\mathcal{A}$, SAM learning iterates over all action triplets in $\mathcal{T}(\lifta)$ 
and, in the worst case, checks every possible parameter-bound literal $\tuple{\liftl, b_{\liftl,\lifta}}$ if it is not a precondition and if it is an effect. 
There are $arity(\lifta, t)^{arity(\liftl, t)}$ ways to bind the parameters of $\liftl$ of type $t$ to the parameters of $\lifta$, and hence $\prod_{t\in T}arity(\lifta, t)^{arity(
\liftl, t)}$ parameter-bound literals with $\lifta$ and $\liftl$. 
\end{proof}

\subsection{Safety Property}

We extend the notion of a \emph{safe action model} to lifted domains as follows. An action model $M$ in a lifted domain is safe iff 
every grounded action defined by $M$ satisfies Eq.~\ref{eq:safe_action_model}. 
This definition preserves the property that a safe action model is an action model that enables generating plans that are guaranteed to be sound w.r.t.\ $\realm$. 
We show next that \sam Learning for lifted domains indeed returns a safe action model. 



\begin{theorem}\label{safe-sam-thm}
Given the injective action binding assumption,
SAM Learning (Algorithm~\ref{alg:sam}) creates a safe action model. 
\end{theorem}
\begin{proof}
We first show by induction on the iterations of the loop in lines~\ref{sam:line:loop-start}--\ref{sam:line:loop-end} that on every iteration
\begin{equation}
    \pre_{\realm}(\lifta)\subseteq\pre(\lifta) \text{ and }
\eff(\lifta)\subseteq\eff_{\realm}(\lifta)
\label{eq:inclusion}
\end{equation}
where $\realm$ is the correct action model. 
Prior to the first iteration, the preconditions of all lifted actions $\lifta$ are all possible parameter-bound literals, so $\pre(\lifta)$ must be a subset of $\pre_{\realm}(\lifta)$.
This includes every parameter-bounded fluent \emph{and} its negation. 
Similarly, the effects are set to $\emptyset$, which is surely a subset of $\eff_{\realm}(\lifta)$.
The changes made to $\pre(\lifta)$ and $\eff(\lifta)$ in subsequent iterations are encapsulated in lines~\ref{line:remove-pre} and~\ref{line:add-eff} in Algorithm~\ref{alg:sam}. 
Line~\ref{line:remove-pre} is a direct application
of Rule 1 (``not a precondition'') from Observation 2, 
and thus $\pre(\lifta)$ is still a subset of $\pre_{\realm}$. 
Similarly, line~\ref{line:add-eff} is an application of Rule 3 (``an effect'') in the same observation, given that $\bindings(b_\lifta, b_\liftl)$ consists of a single parameter binding due to the injective binding assumption. 
This completes the induction.

Let $(\pre, \eff)$ be the action model returned by Algorithm~\ref{alg:sam} (line~\ref{line:return}).  
From the induction above (Eq.~\ref{eq:inclusion}) it immediately follows that for every grounded action $\tuple{\lifta, b_{\lifta}}$ and state $s$, if 
$\pre(\tuple{\lifta, b_{\lifta}})\subseteq s$ then 
$\pre_{\realm}(\tuple{\lifta, b_{\lifta}})\subseteq s$. 
From the induction above, all the parameter-bound literals in $\eff(\lifta)$ are indeed effects of $\lifta$. 
Finally, consider any parameter-bound literal $\tuple{\liftl, b_{\liftl,\lifta}}$
that is an effect of $\lifta$ but is absent from $\eff(\lifta)$, i.e., every 
$\tuple{\liftl, b_{\liftl,\lifta}}\in \eff_{\realm}(A)\setminus\eff(A)$. 
By construction of $\eff$, this can only occur if this parameter-bound literal was true in all pre-states of groundings of $\lifta$ in all the available trajectories. 
Consequently, $\tuple{\liftl, b_{\liftl,\lifta}}$ must be in $\pre(\lifta)$. 
Therefore, every grounded literal in the post-state of applying $\tuple{\lifta, b_{\lifta}}$ in $s$ (i.e., $\tuple{\lifta, b_{\lifta}}_{\realm}(s)$) is either 
in $\eff(\tuple{\lifta, b_{\lifta}})$ or 
$\pre(\tuple{\lifta, b_{\lifta}})$. 
\end{proof}



\subsection{An Example of SAM Learning}

\begin{table}
\small
\begin{tabular}{|l|l|l|l|}
\hline
Action & Params           & Precond.          & Effects               \\\hline
Move   & ?tr - truck      & at(tr, from) & at(tr, to),       \\
       & ?from - location &              & not(at(tr, from)) \\
       & ?to - location   &              &                   \\\hline
Load   & ?pkg - package   & at(tr, loc)  & on(pkg, tr),      \\
       & ?tr - truck      & at(pkg, loc) & not(at(pkg, loc)) \\
       & ?loc - location  &              &                   \\\hline
Unload & ?pkg - package   & at(tr, loc), & not(on(pkg,tr),   \\
       & ?tr - truck      & on(pkg, tr)  & at(pkg, loc)      \\
       & ?loc - location  &              &               \\   \hline
\end{tabular}
\caption{The parameters, preconditions, and effects of the actions according to the real action model of our simple logistics example.}
\label{tab:example-actions}
\end{table}

Consider the following simple logistics problem. 
There are five objects: one \textit{truck} object (tr), one \textit{package} object (pkg), and three \textit{locations} objects ($A$, $B$, and $C$).
\textit{at(?truck, ?location)}  and \textit{on(?truck, ?package)} are lifted fluents representing that the truck is in the location and the package is on the track, respectively. 
There are three possible actions: \textit{Move}, \textit{Load}, and \textit{Unload}. 
Table~\ref{tab:example-actions} lists the parameters, preconditions, and effects of these actions in \realm.
Now, assume we are given three trajectories $T_1$, $T_2$, and $T_3$. 
$T_1$ starts with the truck and the package at location $A$, 
and performs two move actions: Move(tr, A, B) and Move(tr, B, C). 
$T_2$ starts in the same state, but performs Load(pkg, tr, A) and Move(tr, A, B). 
$T_3$ starts with the truck at location $A$ and the package at location $B$, 
and performs Move(tr, A, B), Load(pkg, tr, B), Move(tr, B, C), and Unload(pkg, tr, C). 
Given only the first trajectory $T_1$, the action model returned by SAM Learning already contains the real action model for the lifted Move action, 
since the only grounded fluents that can be bound to the parameters of the grounded action Move(tr, A, B) are 
at(tr, A) and not(at(tr, B)) in the pre-state, 
and at(tr, B) and not(at(tr, A)) in the post-state. 
In contrast, SAM Learning for grounded domains will not know anything about the preconditions and effects of the grounded action Move(tr, B, C) unless it is also given the trajectory $T_3$. 
Similarly, given the second trajectory $T_2$, the action model returned by SAM Learning contains the real action model for the lifted Load action, 
since the only grounded fluents that can be bound to the parameters of the grounded action 
Load(pkg, tr, A) are 
at(tr, A), at(pkg, A), and not(on(pkg, tr)) in the pre-state
and at(tr, A), not(at(pkg, A)), and on(pkg, tr)) in the post-state. 
In fact, given $T_1$, $T_2$, and $T_3$, SAM Learning is able to learn the real action model for this domain. Note that since there are 10 grounded actions in this domain (four Move actions and three Load and Unload actions), SGAM Learning will require at least 10 action triplets to learn an action model with all of the actions. 
\section{Sample Complexity Analysis}

Planning with a safe action model is a sound approach for safe model-free planning, since every plan it outputs is a sound plan according to the real action model. 
However, it is not complete:
a planning problem may be solvable with the real action model, but not the learned one. 
As in prior work on safe model-free planning~\cite{stern2017efficientAndSafe}, we can bound the likelihood of facing such a problem 
as follows. 


Let $\mathcal{P}_D$ be a probability distribution over solvable planning problems in a domain $D$. 
Let $\mathcal{T}_D$ be a probability distribution over pairs $\tuple{P, T}$ 
given by drawing a problem $P$ from $\mathcal{P}(D)$, 
using a sound and complete planner to generate a plan for $P$, 
and setting $T$ to be the trajectory from following this plan.\footnote{%
The planner need not be deterministic.}


\begin{theorem}\label{sam-sample-thm}
Under the injective action binding assumption, given 
$m \geq \frac{1}{\epsilon} (2\ln 3\sum_{\substack{\liftf\in\mathcal{F}\\\lifta\in\mathcal{A}}}\prod_{t\in T}arity(\lifta,t)^{arity(\liftf,t)} + \ln \frac{1}{\delta})$
trajectories sampled from $\mathcal{T}_D$, with probability at least $1-\delta $ 
SAM learning for lifted domains (Algorithm~\ref{alg:sam}) returns a safe action model $M_\sam$ such that a problem drawn from $\mathcal{P}_D$ is not solvable with $M_\sam$ with probability at most $\epsilon$.
\end{theorem}
Theorem~\ref{sam-sample-thm} guarantees
that with high probability ($\geq 1-\delta$) SAM Learning returns an action model that will only fail to solve a given problem with low probability ($\leq \epsilon$), given a number of example trajectories linear in the size of the models.
For example, in the real action model of our simple logistics example with two binary fluents and three ternary actions, the load and unload actions have a single argument of each type; only the move action has two arguments of the same type (location). The only fluents that have location arguments are the at fluents, which have arity one with respect to locations. Thus, guaranteeing $\epsilon=\delta=5\%$ requires only 324 trajectories.
The rest of this section is devoted to establishing Theorem~\ref{sam-sample-thm}. 

\begin{definition}[Adequate]
An action model $M$ is {\em $\epsilon$-adequate} if, with probability at most $\epsilon$, a trajectory $T$ sampled from $\mathcal{T}_D$ contains an action triplet $\tuple{s,a,s'}$ where 
$s$ does not satisfy $\pre_M(a)$.\footnote{An action model may not contain any information about some action $a$. 
For the purpose of safe planning this is equivalent to an action model in which the precondition to $a$ can never be satisfied.}
\end{definition}

\begin{lemma}\label{adequate-lem}
The action model returned by SAM Learning (Algorithm~\ref{alg:sam}) 
given $m$ trajectories (as specified in Theorem \ref{sam-sample-thm})
is $\epsilon$-adequate with probability at least $1-\delta$. 
\end{lemma}
\begin{proof}
Consider any action model $M_B$ that may be returned by SAM Learning but is not $\epsilon$ adequate. 
By definition, the probability of drawing a trajectory from $\mathcal{T}_D$ that is inconsistent with $M_B$ is at least $\epsilon$. Thus, the probability of drawing $m$ samples 
that are consistent with $M_B$ is at most
\begin{equation}
    (1-\epsilon)^{m}\leq e^{-m\cdot\epsilon}.
\end{equation}
$M_B$ can only be returned if this occurs. For our choice of $m$,
\begin{equation}
e^{-m\cdot \epsilon }
\leq e^{-(\ln{3}L + \ln\frac{1}{\delta})} 
= \frac{\delta}{3^L}
\label{eq:singleBadModelProb}
\end{equation}
where 
\[
L=2\sum_{\substack{\liftf\in\mathcal{F}\\\lifta\in\mathcal{A}}}\prod_{t\in T}arity(\lifta,t)^{arity(\liftf,t)}
\]
Let $B$ be the set of action models that are not $\epsilon$-adequate. 
By a union bound over $B$, the probability that SAM Learning will return an action model that is not $\epsilon$-adequate 
is at most $\frac{|B|\delta}{3^L}$.
For each parameter-bound fluent, each precondition or effect will either contain that fluent, or its negation, or neither of them. Hence, the number of possible action models is $3^L$. 
Since $B$ is a set of action models, we have that the size of $B$ is at most $3^{L}$. 
Therefore, the probability that SAM Learning will return an action model that is not $\epsilon$-adequate is at most $\delta$. 
\end{proof}

\noindent Finally, we can prove Theorem~\ref{sam-sample-thm} as follows. 
\begin{proof}
Let $M$ be an action model returned by SAM Learning given $m$ samples. 
Thus, $M$ is a safe action model (Theorem~\ref{safe-sam-thm}) and it is $\epsilon$ adequate (Lemma~\ref{adequate-lem}). 
Consider a problem $P$ drawn from $\mathcal{P}(D)$, 
and its corresponding pair $\tuple{P, T}$ from $\mathcal{T}(D)$. 
Since $M$ is $\epsilon$-adequate, with probability at least $1-\epsilon$, 
for every action triplet $\tuple{s,a,s'}\in T$ 
$a$ is applicable in $s$, that is, $\pre_M(a)\subseteq s$. 
Since $M$ is a safe action model, we have that $a_M(s)=a_{\realm}(s)=s'$. 
Thus, with probability at least $1-\epsilon$ the trajectory $T$ is consistent with the learned action model $M$, and therefore $P$ can be solved with $M$ 
\end{proof}

\section{Multiple Action Bindings}
When the injective action-binding assumption does not hold, 
multiple action parameters are bound to the same object and thus $(b_\lifta)^{-1}$ is not defined. 
As a result, when SAM Learning infers an effect (Rule 3 in Observation~\ref{obs:sam-rules-lifted-general}) it cannot generalize it to be a unique effect of the corresponding lifted action, as done in line~\ref{line:infer-effect} in Algorithm~\ref{alg:sam}. 
This poses a challenge to learning a safe action model, as the information that can be inferred from observing action triplets can be complex.

For example,
consider a lifted action $\lifta(x,y)$.  
Suppose $x$ and $y$ are associated with the same type and $o$ is an object of that type. 
Given the action triplet 
$\tuple{\{~\}, \lifta(o, o), \{\liftl(o)\}}$, 
the agent can infer that $\liftl(o)$ is an effect of the grounded action $\lifta(o, o)$. 
However, the agent cannot accurately infer the effect of the lifted action $\lifta(x,y)$: it can be
 either $\{\liftl(x)\}$, $\{\liftl(y)\}$, or both. 
 Concretely, if $o_1$ and $o_2$ are two different objects from the same type as $o$, 
 the agent cannot determine if applying $\lifta(o_1,o_2)$ will result in a state with
 $\{\liftl(o_1)\}$, $\{\liftl(o_2)\}$, or $\{\liftl(o_1), \liftl(o_2)\}$.
Consequently, any safe action model must not enable groundings of $\lifta$ that bind $x$ and $y$ to different objects, unless $L(x)$ and $L(y)$ both already hold. 

Now, assume the agent is also given the action triplet 
$\tuple{\{\liftl(o_1)\}, \lifta(o_1, o_2), \{\liftl(o_1)\}}$. 
The pre- and post-state are the same, so in Algorithm~\ref{alg:sam} we cannot learn any new effects of $\lifta$ from this triplet. However, we can infer that $\liftl(o_2)$ is not an effect of the grounded action in this triplet. Consequently, the parameter-bound literal $\liftl(y)$ cannot be an effect of the lifted action $\lifta$. 
Thus, this second action triplet does provide useful information: it allow us to infer that 
the lifted action $\lifta(x,y)$ has a parameter-bound effect $\liftl(x)$. 

In a planning task, we might avoid the above by reformulating the domain to satisfy the injective action binding assumption. However, in a learning setup, we do not have control over how the domain is formulated and so the domain we are learning may indeed violate the injective action binding assumption, preventing the application of SAM learning. Next, we describe Extended SAM Learning, which addresses such cases by capturing the form of inference described above.

\subsection{Extended SAM Learning}
Extended SAM (E-SAM) learning works in two stages. First, it creates 
for every lifted action $\lifta$ a conjunction and a Conjunctive Normal Form (CNF) formula, denoted $\conj_\pre(\lifta)$ and $\cnf_\eff(\lifta)$, that describe a set of constraints for a safe action model. 
Then E-SAM learning generates a safe action model based on these formulas.

\subsubsection{Safe Action Model Constraints}
$\conj_\pre(\lifta)$ uses atoms of the form 
\ispre($\tuple{\liftl,b_{\liftl,\lifta}}$), which specify that 
$\tuple{\liftl,b_{\liftl,\lifta}}$ is a precondition $\liftl$ in a safe action model. 
Similarly, $\cnf_\eff(\lifta)$ uses atoms of the form 
\iseff($\tuple{\liftl,b_{\liftl,\lifta}}$), 
which specify that $\tuple{\liftl,b_{\liftl,\lifta}}$ is an effect of $\liftl$ in a safe action model. 

Initially, $\conj_\pre(\lifta)$ and $\cnf_\eff(\lifta)$ 
represent that all possible parameter-bound literals are preconditions and there are no effects. 
Then, E-SAM learning iterates over every action triplet $(s, a, s')$ in the given set of trajectories in which $a$ is a grounding of $\lifta$. 
For every such triplet, it applies the inference rules in Observation~\ref{obs:sam-rules-lifted-general} as follows. 

Every parameter-bound literal $\tuple{\liftl,b_{\liftl, \lifta}}$ 
such that $\tuple{\liftl, b_\lifta\circ b_{\liftl,\lifta}}$ is 
not in the pre-state cannot be a pre-condition (Rule 1). So, we remove \ispre$(\tuple{\liftl,b_{\liftl, \lifta}})$ from $\conj_\pre$ for such parameter-bound literals.  
Similarly, every parameter-bound literal $\tuple{\liftl,b_{\liftl, \lifta}}$ 
such that $\tuple{\liftl, b_\lifta\circ b_{\liftl,\lifta}}$ is 
not in the post-state cannot be an effect (Rule 2). So, we add $\neg$\iseff$(\tuple{\liftl,b_{\liftl, \lifta}})$ to $\cnf_\eff$ for such parameter-bound literals.  
Finally, every grounded literal $\tuple{\liftl, b_\liftl}$ in $s'\setminus s$ must be an effect. 
So, we add to $\cnf_\eff$ the disjunction over 
all parameter-bound literals
$\tuple{\liftl, b_\lifta\circ b_{\liftl,\lifta}}$ that satisfy $\tuple{\liftl, b_\lifta\circ b_{\liftl,\lifta}}=\tuple{\liftl, b_\liftl}$ (Rule 3).
Once the given trajectories have been processed by the algorithm, 
we simplify the CNF by applying unit propagation and removing subsumed clauses. 



\subsubsection{Proxy Actions}
The main challenge in creating a safe action model from the generated formulas is the disjunction in $\cnf_\eff$, which represents uncertainty w.r.t to the effects of action. 
To address this, we create a safe action model with a set of proxy actions that ensure every action is only applicable when we know its effects. We achieve this by computing, for each possible subset of the parameter-bound literals in the formulas for a given action, most general unifiers for the literals; in our setting, such unifiers simply identify subsets of the action parameters. Alternatively, if the parameter-bound literal appears in the precondition, then the literal does not need to be included in the unifier, and we know that the corresponding effect will always hold. Hence, since at least one of the unified parameters occurs in the parameter binding of the effect in the true action model, so when the parameters in the set are all bound to the same object (or the literal appears in the precondition), we can guarantee that the corresponding effect literal holds in the post-state.  In more detail, this is done as follows. 

If an action has only unit clauses, we have a single action with the effects indicated by the positive literals. Otherwise, we create a proxy action for all subsets of the parameter-bound literals in non-subsumed non-unit clauses. (The number of proxy actions is thus exponential in the size of the formula of non-unit clauses.) In this proxy action, we identify all of the parameters that appear in the same position of the literals in the subset with the same fluents. Each proxy action has the following set of preconditions and effects: every unit clause in the CNF and every clause in the corresponding subset specifies an effect of the proxy action. For the subset of literals not chosen for this proxy action, the proxy action has the corresponding literals as additional preconditions, in addition to the preconditions of the original SAM Learning action model. 
Every plan generated by the action model created by the resulting action model is translated to a plan without proxy actions by replacing them with the actions for which they were created. 
Algorithm~\ref{alg:esafeplan} and~\ref{alg:extract-clauses} lists the complete pseudocode of E-SAM learning.

\begin{algorithm}[t]
\footnotesize
\DontPrintSemicolon
\SetKwInOut{Input}{Input}\SetKwInOut{Output}{Output}
\Input{$\Pi_\mathcal{T} = \tuple{T,O, s_I, G, \mathcal{T}}$}
\Output{$(\pre, \eff)$ for a safe action model}
\BlankLine
    $\mathcal{A'}\gets$ all lifted actions observed in $\mathcal{T}$\\
    \ForEach{lifted action $\lifta\in \mathcal{A'}$}{
        $(\conj_\pre, \cnf_\eff)\gets$ ExtractClauses($\lifta$, $\mathcal{T}(\lifta)$) \nllabel{line:esam:extract-clauses}\\
        $\cnf^1_\eff\gets$ all unit clauses in $\cnf_\eff$\\
        SurelyEff$\gets \{ l ~|~ \text{IsEff}(l)\in \cnf^1_\eff\}$\\
        SurelyPre$\gets \{ l ~|~ \text{IsPre}(l)\in \conj_\pre \}$\\ 
        \tcc{Create proxy actions for non-unit effects clauses}
        $\cnf_\eff\gets \cnf_\eff\setminus \cnf^1_\eff$ \\
        \ForEach{$S\in$ Powerset($\cnf_\eff$)}{
            $\pre(\lifta_S)\gets $ SurelyPre; $\eff(\lifta_S)\gets $ SurelyEff \\
            \ForEach{$C_\eff\in \cnf_\eff\setminus S$}{
                \ForEach{$\text{IsEff}(l)\in C_\eff$}{
                    Add $l$ to $\pre(\lifta_S)$
                }
            }
            MergeObjects$\big(S, \pre(\lifta_S),\eff(\lifta_S)\big)$
        }
    } 
    \Return $(\pre, \eff)$ 

\caption{Extended SAM Learning}\label{alg:esafeplan}
\end{algorithm}

\begin{algorithm}
\DontPrintSemicolon
\SetKwInOut{Input}{Input}\SetKwInOut{Output}{Output}
\Input{$\lifta$, a lifted action}
\Input{$\mathcal{T}(\lifta)$, action triplets that contain $\lifta$}
\Output{$(\conj_\pre, \cnf_\eff)$, representing the constraints over $\pre(\lifta)$ and $\eff(\lifta)$}
\BlankLine
    $\cnf_\eff \gets\emptyset$; $\conj_\pre\gets \emptyset$ \nllabel{line:init}\\ 
    \ForEach{parameter-bound literal $\tuple{L,b_{\liftl,\lifta}}$}{
        Add \ispre($\tuple{L,b_{\liftl,\lifta}}$) to $\conj_\pre$
    }
    
    \ForEach{$(s, \tuple{\lifta,b_\lifta}, s')\in \mathcal{T}(\lifta)$}{
        \ForEach{\ispre($\tuple{L,b_{\liftl,\lifta}}\in \conj_\pre$)}{
            \If{$\tuple{L, b_\lifta\circ b_{\liftl,\lifta}}\notin s$}{
                Remove \ispre($\tuple{L,b_{\liftl,\lifta}}$) from $\conj_\pre$ \nllabel{line:rule1}
            }
        }
    
        \ForEach{$\tuple{L,b_\liftl}\in s'\setminus s$}{
            $C_\eff\gets\bot$\\
            \ForEach{$b_{\liftl,\lifta}\in\bindings(b_\lifta, b_\liftl)$}{
                $C_\eff \gets C_\eff \vee$ \iseff($\tuple{\liftl, b_{\liftl,\lifta}})$
            }
            Add EffectsClause to $\cnf_\eff$ \nllabel{line:rule2}
        }
        \ForEach{parameter bound literal $\tuple{L,b_{\liftl,\lifta}}$}{
            $b_\liftl\gets \tuple{L,b_{\liftl,\lifta}\circ b_\lifta}$\\
            \If{$\tuple{L,b_\liftl}\notin s'$}{
                Add $\neg$\iseff($\tuple{L,b_{\liftl,\lifta}}$) to $\cnf_\eff$ \nllabel{line:rule3}
            }
        }
    } 
    Minimize($\cnf_\eff$)~\nllabel{line:minimize}\\
    \Return $(\conj_\pre, \cnf_\eff)$

\caption{ExtractClauses}
\label{alg:extract-clauses}
\end{algorithm}

\subsubsection{Theoretical Properties}
E-SAM Learning creates an action model that satisfies the same properties as the action model created by SAM learning under the injective action binding assumption, as captured in Theorems~\ref{safe-sam-thm} and~\ref{thm:sam-learning-complete-grounded}. 

\begin{theorem}\label{extended-sam-safe}
The E-SAM Learning action model is safe.
\end{theorem}
\begin{proof}
For each of the proxy actions, for every effect, at least one of the parameter-bound literals for the identified parameters is an effect of the true action. Furthermore,  the preconditions ensure that the rest of the uncertain effects are already present in the pre-state. The post-state of the proxy action is thus identical to that of the true action when its precondition is satisfied. Likewise, the proxy actions have preconditions that are only stronger than the actual precondition. Eq.~\ref{eq:safe_action_model} therefore holds. The rest of the claim now follows from the argument in Theorem~\ref{safe-sam-thm}.
\end{proof}

Recall that a \emph{prime implicate} is a clause that is entailed by a formula for which no subclause is also entailed. 
$\cnf_\eff$ consists of precisely these prime implicates.
\begin{lemma}\label{unit-prop-complete-lem}
All prime implicates of $\cnf_\eff$ 
are derived by unit propagation.
\end{lemma}
\begin{proof}
Note that the clauses created by Rule 3 contain only positive literals, and negative literals are only created by Rule 1 and 2, which create unit clauses. Hence, unit propagation is sufficient to capture all possible resolution inferences from these clauses. 
By the completeness of resolution for prime implicates (e.g., \cite[Ch.\ 13, Exercise 1]{brachman2004}),
all of the prime implicates of $\cnf_\eff$ 
can be derived by resolution. In turn, therefore, unit propagation can also derive all of the prime implicates of $\cnf_\eff$.
\end{proof}

\begin{theorem}\label{e-sam-strong-theorem}
Every action model $M'$ that is consistent with $\mathcal{T}$ 
and safe w.r.t.\ the real action model \realm is also safe with respect to the extended SAM Learning action model.
\end{theorem}
\begin{proof}
Let $M'$ be an action model that is consistent with $\mathcal{T}$ and safe w.r.t.\ \realm, and let $\tuple{s,\tuple{\lifta,b},s'}$ be an action triplet permitted by $M'$. 
Consider the set $S$ of literals in $s'\setminus s$ that do not correspond to unit clauses in the CNF created by E-SAM Learning, and the set $\bar{S}$ of literals that are the groundings under $b$ of the effects in the non-unit clauses created by E-SAM learning that are not in $s'\setminus s$. 
Recall, Observation~\ref{obs:sam-rules-lifted-general} characterizes the set action models consistent with $\mathcal{T}$ and by Lemma~\ref{unit-prop-complete-lem}, no sub-clause of the CNF created by E-SAM learning is entailed by the rules of Observation~\ref{obs:sam-rules-lifted-general}. 
Therefore, for every literal of every non-unit clause of this CNF, there exists an action model consistent with $\mathcal{T}$ in which that literal is the only satisfied literal of the clause. (Otherwise, a strictly smaller clause would be entailed.) Therefore, for each literal $l\in S$, since $M'$ is safe w.r.t.\ \realm, all of the parameters of $\lifta$ in some clause for this effect must be bound to the objects necessary to obtain $l$ as the corresponding effect. Thus, $b$ must be consistent with at least one of the proxy actions $\lifta_{proxy}$. Furthermore, since the literals in $\bar{S}$ may be effects of $\tuple{\lifta, b}$, if they are not in $s'\setminus s$, they must be in $s$, so the preconditions of $\lifta_{proxy}$ are satisfied as well. Since the E-SAM Learning action model is safe by Theorem~5, the post-state of $\lifta_{proxy}$ is therefore equal to that obtained by the true action model, which is in turn also equal to $s'$ since $M'$ is also safe. $\lifta$ is therefore an application of $\lifta_{proxy}$, and we see that the use of $\lifta$ in $M'$ is safe with respect to the set of proxy actions in the action model created by E-SAM Learning.
\end{proof}

\paragraph{Time complexity}
E-SAM learning can be split into two parts: a \textbf{learning} part, which extracts clauses about the preconditions and effects of the actions (line~\ref{line:esam:extract-clauses} in Alg.~\ref{alg:esafeplan}), and a \textbf{compilation} part that generates a PDDL encoding that can be used by off-the-shelf PDDL planners. The latter involves the creation of the proxy actions. 
The learning part of E-SAM learning can be implemented to run in polynomial time in the number of parameter-bound literals and total number of action triplets, similar to SAM learning. 
The compilation part, however, may run in exponential time due to the 
 inability of PDDL to capture the uncertainty over actions' effects that has been learned (captured by $\cnf_\eff$). 
 Future work may investigate avoiding this exponential step by instead compiling the learned knowledge to a domain encoding for a conformant planner~\cite{bonet2010conformant}.

\section{Experiments}
Next, we perform an experimental evaluation of SAM Learning over planning problems 
from twelve domains from the IPC~\cite{ipc}. 
Table~\ref{tab:domains} lists the names of these domain, the number of lifted fluents and actions, the largest arity of these lifted fluents and actions, and the largest number of grounded fluents and actions in our dataset. 
We have chosen only domains in the IPC benchmarks in which the injective action binding assumption holds. 
For such domains, E-SAM Learning and SAM Learning behave the same. 

\begin{table}[t]
\centering
\resizebox{\columnwidth}{!}{
\begin{tabular}{lccccrr}
\toprule
            &               \#&             \# &        max &       max &       max &       max \\
            & lifted  & lifted & arity & arity & ground & ground\\
            &fluents&actions&fluents&actions & fluents&actions \\ 
\midrule

Blocks      & 5                 & 4                 &2                  &2  & 182 & 182                  \\
Depot      & 6                 & 5                 &4                  &2  & 75 & 450                  \\
Ferry       & 5                 & 3                 &2                  &2& 75 & 75                    \\
Floortile   & 10                & 7                 &2                  &4 & 40 & 64                   \\
Gripper     & 4                 & 3                 &2                  &3& 42 & 84                    \\
Hanoi       & 3                 & 1                 &2                  &3& 33 & 166                    \\
Npuzzle     & 3                 & 1                 &2                  &3& 80 & 80                    \\
Parking     & 5                 & 4                 &2                  &3  & 182 & 2,184                  \\
Satellite      & 8                 & 5                 &2                  &4  & 75 & 1875                  \\
Sokoban     & 4                 & 2                 &3                  &5& 288 & 564                    \\
Spanner     & 6                 & 3                 &2                  &4  & 12 & 12                  \\
Transport      & 5                 & 3                 &2                  &5  & 870 &  3600                \\

\bottomrule
\end{tabular}
}
\caption{Statistics on the domains in our experiments.}
\label{tab:domains}
\end{table}

\begin{table}[bt!] 
\centering
\resizebox{0.99\columnwidth}{!}{
\begin{tabular}{ll|cc|cc}
\toprule
                                &                                   & \multicolumn{2}{c}{Trajectories}                   & \multicolumn{2}{c}{Triplets}                       \\ 
Domain                          & \# Objects                        & \multicolumn{1}{c}{SAM}& \multicolumn{1}{c}{FAMA} & \multicolumn{1}{c}{SAM}& \multicolumn{1}{c}{FAMA} \\ \midrule
\multirow{8}{*}{Blocks} & 7 blocks                                & 1                       & 2                        & 13                      & 22                       \\
                                & 8 blocks                                & 1                       & 2                        & 16                      & 29                       \\
                                & 9 blocks                                & 1                       & 2                        & 18                      & 35                       \\
                                & 10 blocks                                & 1                       & 2                        & 22                      & 40                       \\
                                & 11 blocks                                & 1                       & 2                        & 25                      & 46                       \\
                                & 12 blocks                                & 1                       & 2                        & 28                      & 53                       \\
                                & 13 blocks                                & 1                       & 2                        & 35                      & 60                       \\
                                & 14 blocks                                & 1                       & 2                        & 42                      & 72                       \\ \midrule
\multirow{4}{*}{Depot}       & 1 truck, 2 places, 4 hoists, 10 crates               & 1                       & 1                        & 18                       & 24                        \\
                                & 1 truck, 2 places, 4 hoists, 15 crates              & 1                       & 1                        & 26                       & 32                       \\
                                & 2 trucks, 3 places, 5 hoists, 10 crates                & 1                       & 1                        & 22                      & 28                       \\
                                & 2 trucks, 3 places, 5 hoists, 15 crates              & 1                       & 1                        & 28                      & 36                       \\ \midrule
\multirow{4}{*}{Ferry}       & 2 locations, 8 cars               & 1                       & 1                        & 4                       & 7                        \\
                                & 3 locations, 10 cars              & 1                       & 1                        & 9                       & 12                       \\
                                & 4 locations, 12 cars              & 1                       & 1                        & 12                      & 15                       \\
                                & 5 locations, 15 cars              & 1                       & 1                        & 14                      & 17                       \\ \midrule
\multirow{4}{*}{Floortile}   & 3x3, 2 robots                     & 1                       & 2                        & 13                      & 22                       \\
                                & 4x3, 2 robots                     & 1                       & 2                        & 13                      & 32                       \\
                                & 4x4, 2 robots                     & 1                       & 2                        & 16                      & 40                       \\
                                & 5x4, 2 robots                     & 1                       & 2                        & 16                      & 52                       \\ \midrule
\multirow{4}{*}{Gripper}        & 2 rooms, 6 balls                  & 1                       & 1                        & 4                       & 8                        \\
                                & 2 rooms, 10 balls                 & 1                       & 1                        & 5                       & 8                        \\
                                & 3 rooms, 8 balls                  & 1                       & 1                        & 5                       & 8                        \\ 
                                & 3 rooms, 14 balls                 & 1                       & 1                        & 5                       & 9                        \\ \midrule
\multirow{4}{*}{Hanoi}       & 3 disks                           & 1                       & 1                        & 3                       & 3                        \\
                                & 4 disks                           & 1                       & 1                        & 3                       & 3                        \\
                                & 5 disks                           & 1                       & 1                        & 3                       & 3                        \\
                                & 6 disks                           & 1                       & 1                        & 3                       & 3                        \\ \midrule
\multirow{3}{*}{Npuzzle}     & 8 tiles                                 & 1                       & 1                        & 1                       & 1                        \\
                                & 15 tiles                                & 1                       & 1                        & 1                       & 1                        \\
                                & 24 tiles                                & 1                       & 1                        & 1                       & 1                        \\ \midrule
\multirow{4}{*}{Parking}     & 3 curbs, 4 cars                   & 2                       & 3                        & 13                      & 20                       \\
                                & 5 curbs, 8 cars                   & 2                       & 4                        & 32                      & 52                       \\
                                & 7 curbs, 12 cars                  & 2                       & 4                        & 53                      & 87                       \\
                                & 8 curbs, 14 cars                  & 2                       & 3                        & 72                      & 98                       \\ \midrule
\multirow{4}{*}{Satellite}     & 2 sats., 4 instrs., 4 modes, 8 dirs.                      & 1                       & 1                        & 20                       & 28                        \\
                                & 4 sats., 4 instrs., 4 modes, 8 dirs.                        & 1                       & 1                        & 20                       & 28                        \\
                                & 5 sats., 5 instrs., 5 modes, 10 dirs.                        & 1                       & 1                        & 24                       & 32                        \\
                                & 5 sats., 5 instrs., 5 modes, 15 dirs.                      & 1                       & 1                        & 26                       & 34                       \\ \midrule
\multirow{4}{*}{Sokoban}     & 5x5, 2 boxes                      & 1                       & 1                        & 4                       & 6                        \\
                                & 7x7, 2 boxes                      & 1                       & 1                        & 6                       & 8                        \\
                                & 8x8, 3 boxes                      & 1                       & 1                        & 5                       & 7                        \\
                                & 9x9, 3 boxes                      & 1                       & 1                        & 8                       & 10                       \\ \midrule
\multirow{9}{*}{Spanner}     & 10 spanners, 10 nuts, 2 locations & 1                       & 1                        & 14                      & 16                       \\
                                & 10 spanners, 10 nuts, 4 locations & 1                       & 1                        & 16                      & 18                       \\
                                & 10 spanners, 10 nuts, 6 locations & 1                       & 1                        & 18                      & 20                       \\
                                & 11 spanners, 11 nuts, 2 locations & 1                       & 1                        & 15                      & 17                       \\
                                & 11 spanners, 11 nuts, 4 locations & 1                       & 1                        & 17                      & 19                       \\
                                & 11 spanners, 11 nuts, 6 locations & 1                       & 1                        & 19                      & 21                       \\
                                & 12 spanners, 12 nuts, 2 locations & 1                       & 1                        & 16                      & 18                       \\
                                & 12 spanners, 12 nuts, 4 locations & 1                       & 1                        & 18                      & 20                       \\
                                & 12 spanners, 12 nuts, 6 locations & 1                       & 1                        & 20                      & 22                       \\ \midrule
\multirow{4}{*}{Transport}     & 2 trucks, 5 packages, 10 locations                      & 1                       & 1                        & 16                       & 20                        \\
                                & 2 trucks, 10 packages, 20 locations                     & 1                       & 1                        & 18                       & 22                        \\
                                & 4 trucks, 10 packages, 20 locations                      & 1                       & 1                        & 22                       & 26                        \\
                                & 4 trucks, 15 packages, 30 locations                    & 1                       & 1                        & 24                       & 30                       \\ \bottomrule
\end{tabular}
}
\caption{Number of trajectories and action triplets needed to learn the real action model in each domain.} 
\label{tab:all_experiments}
\end{table}

For each domain, we generated problems using the problem generator provided in the IPC learning tracks 
and solved them using their true action models with the MADAGASCAR planner~\cite{Rintanen2014MadagascarS} to obtain example trajectories.
These trajectories were broken to action triplets and given to the SAM Learning algorithm one at a time to obtain a safe action model. We halted this process when the learned action model was equivalent to the real model, and report the number of triplets and trajectories given to the algorithm. 
As a baseline, we performed this experiment also with FAMA \cite{aineto19}, which is a modern algorithm for learning action models from trajectories. Note that unlike SAM Learning, FAMA has no safety guarantee. 
In addition, SAM learning runs in time linear in the number of lifted actions, lifted literals, and trajectories, while FAMA runs an automated planner which has an exponential worst-case running time (as planning is PSPACE-complete). SAM learning is only exponential in the maximal number of parameters of each action and literal. Thus, SAM learning can easily scale to very large domains.

Table~\ref{tab:all_experiments} lists the results of our experiments. The ``\# Objects'' column lists the objects in the problem, and the values under ``Trajectories'' and ``Triplets'' are the number of trajectories and action triplets, respectively, required to learn the correct model. 
In all cases, both methods were able to recover the real action model. 
However, SAM Learning was able to find such a model using at most as many, and often significantly fewer triplets and trajectories. 
For example, for the Floortile problem with 2 robots and a $5\times 4$ floor, SAM learned the correct model with only 16 action triplets while FAMA required 52 action triplets. In fact, in all domains except Parking, SAM Learning learned the correct model with a single trajectory.  
Note that once SAM Learning finds a correct model it will never change it,  since SAM only removes literals that are not satisfied in the pre-state from the preconditions and adds literals that switch values between pre and post-states to the effects. Meanwhile, FAMA might add irrelevant literals or remove correct literals from the preconditons or effects as it processes more action triplets. 


The code for SAM learning and our experiments is available at \url{https://github.com/hsle/sam-learning}.


\section{Related Work}

A variety of notions of \emph{safety} have been considered in RL, for example capturing the ability to reliably return to a home state \cite{moldovan2012safe} or avoiding undesirable states  \cite{turchetta2016,wachi2018safe} while learning about the environment. 
But, these approaches to safe exploration require some kind of strong prior knowledge, either in the form of beliefs about the transition model or knowledge that the safety levels 
follow a Gaussian process model. 
Such assumptions are reasonable in the low-level motion planning tasks where RL excels, but they do not suit the kind of discrete, high-level problems typically considered in domain-independent planning.
In addition, in these works safety is soft constraint that an algorithm aims to maximize, while in our case safety is a hard constraint.

Our work is part of the growing literature on learning action models for domain-independent planning~\cite{arora2018review}, which includes algorithms such as ARMS~\cite{yang2007learning}, LOCM~\cite{cresswell2013acquiring}, LOCM2~\cite{cresswell2011generalised}, AMAN~\cite{zhuo2013action}, and FAMA~\cite{aineto19}. 
Similar to SAM learning, ARMS~\cite{yang2007learning} also defines rules to infer an action model from a given set of trajectories. 
Our third rule (``must be an effect'') is somewhat similar to their (I.1) rule. 
The other ARMS rules are different, and are designed to \textbf{explain} the observed trajectories in a succinct manner. Thus, the action model created by ARMS may be either under of over constrained for this purpose. 
LOCM~\cite{cresswell2013acquiring} and LOCM2~\cite{cresswell2013acquiring} 
learn action models by learning and composition state transitions that are consistent with the observed trajectories. They do not require as input a set of possible lifted fluents or types, and learn from data the relation between objects and types, as well as the relation between actions and types. 
LOCM2 is a heuristic version of LOCM algorithm for learning action models. It does not support domains in which the injective action binding assumption does not hold, or domains where a deleted literal does not appear as a precondition. 
AMAN~\cite{zhuo2013action} is an action-model learning algorithm that is specifically designed to handle noisy observations. It constructs a graphical model and learns the statistical relationship between actions and possible state transitions. 
FAMA~\cite{cresswell2011generalised} compiles the problem of finding an action model that is consistent with a set of trajectories to a planning problem. The solution to this planning problem is a sequence of ``actions'' that construct an action model. FAMA is more general than SAM or ESAM in the sense that it supports partial observability. 

FAMA, as well as LOCM, LOCM2, and AMAN, aim to create an action model that \textbf{explains} the given trajectories. This can be viewed as a solving an \emph{inductive logic programming}~\cite{muggleton1994inductive} task. 
The action model they generated is only guaranteed to be \textbf{consistent} with the given set of observations.
Our algorithms (SAM and ESAM) provide a stronger guarantee: the action model they create is \textbf{safe} with respect to the real action model ($\realm$). 
A safe action model (Definition~\ref{def:safe_action_model}) is, by definition, consistent with the given trajectories, 
but a consistent action model (Definition~\ref{def:consistent}) may very well be unsafe. For example, consider an action model $M$ in which the effects of all actions are correct (i.e., match the effects in $\realm$) and all actions have no preconditions. This action model is clearly unsafe, and plans generated with it may be not sound. Yet, such an action model is consistent with any trajectory generated according to the real action model ($\realm$). 
None of the works listed above provide a safety guarantee, and plans generated with the action models they generate may be unsound.\footnote{Note that the soundness and completeness of FAMA (Lemmas 1 and 2 there) do not refer to plans generated by the action model FAMA learns, but to the learning algorithm it self. 
That is, FAMA is sound in the sense that the action model it returns is consistent with the given trajectories, and it is complete in the sense that if a consistent action model exists then FAMA will find it. Indeed, FAMA may return an unsafe action model.}

\section{Relaxing the Assumptions}

Our learning of preconditions is very similar to Valiant's elimination algorithm \cite{Valiant1984} for learning conjunctions in supervised learning. 
Following his work, we can easily support preconditions that are $k$-CNFs (and not just a simple conjunction) 
by considering that all sets of possible clauses of size $k$ as preconditions instead of a simple conjunction. 
This will increase the sample complexity bound in Theorem~\ref{sam-sample-thm} by raising the first term to the $k^{\text{th}}$ power and similarly increase the running time.
Conditional effects can be similarly supported if we can bound the number of literals in their firing condition by some value $k$. 
In this case, the extension to SAM learning keeps track of all possible conditions with at most $k$ literals that hold when an action is applied.


We believe that the algorithm can similarly be extended to handle independent, random noise, provided that either (a) all fluents are corrupted with the same probability or (b) the rate of corruption of each fluent is known. Indeed, this is an example of independent attribute noise, and extensions of Valiant's elimination algorithm to these settings were proposed by Goldman and Sloan~\shortcite{GoldmanS1995} (extending Shackleford and Volper~\citeyear{ShackelfordV1988})
and Decatur and Gennaro~\shortcite{DecaturG1995}, respectively. In the presence of such noise, however, the safety property must be weakened: indeed, since any combination of fluent settings may be observed, albeit with exponentially small probability, we can only expect to guarantee that the action model will be safe with high probability w.r.t.\ the noise. 
Likewise, we believe similar guarantees are possible in sufficiently benign partial information settings, following Michael~\shortcite{michael2010partialObservability}.

If the environment itself is far from deterministic, then clearly the STRIPS rules we learn would be inappropriate, and a different representation would be necessary. We note that if the environment is stochastic and there is noise of unknown rates that differ across fluents, then it seems to be information-theoretically impossible to learn a safe model (in our sense) even when an adequate set of deterministic rules exists, cf.\ the counterexample of Goldman and Sloan~\shortcite{GoldmanS1995}: we cannot distinguish between a fluent that is just corrupted by observation noise from a fluent that is merely correlated with it.

\section{Conclusion and Future Work}

In this work, we presented the Safe Action Model Learning algorithm for lifted domains.
SAM Learning for lifted domains is guaranteed to return an action model that produces sound plans, 
even without knowing the preconditions and effect of the actions in the domain.   
A theoretical analysis shows that the number of trajectories needed to learn an action model that will solve a given problem with high probability is linear in the potential size of the action model. 
This approach is suitable for most domains in current planning benchmarks, 
where the effects of actions are trivial unless the action parameters are bound to different objects. 
We also discussed how to adapt our algorithm to the case where this assumption does not hold. 
In the future, we aim to 
extend safe action-model learning to domains with partial observability and stochasticity. 

\section*{Acknowledgements}
This research is partially funded by NSF awards IIS-1908287, IIS-1939677, and CCF-1718380,
and BSF grant \#2018684 to Roni Stern.

\bibliographystyle{kr}
\bibliography{library}
\end{document}